\documentclass[runningheads]{llncs}
\usepackage{graphicx}
\usepackage{amsmath,amssymb} 
\usepackage{color}\usepackage[width=122mm,left=12mm,paperwidth=146mm,height=193mm,top=12mm,paperheight=217mm]{geometry}
\usepackage{url,xspace,caption,subcaption}
\usepackage{array,multirow,booktabs,balance}
\usepackage{times,txfonts}
\usepackage{hyperref}  
\hypersetup{backref,colorlinks=true,linkcolor=red,citecolor=green,urlcolor=blue}

\makeatletter
\DeclareRobustCommand\onedot{\futurelet\@let@token\@onedot}
\def\@onedot{\ifx\@let@token.\else.\null\fi\xspace}

\def\eg{\emph{e.g}\onedot} 
\def\ie{\emph{i.e}\onedot}

\def\etal{\emph{et al}\onedot}
\makeatother

\def\Vec#1{{\boldsymbol{#1}}}
\def\Mat#1{{\boldsymbol{#1}}}

\def\GRASS#1#2{\mathcal{G}({#1},{#2})}

\newcommand{\tr}{\mathop{\rm  Tr}\nolimits}

\begin{document}
\pagestyle{headings}
\mainmatter
\title{Expanding the Family of Grassmannian Kernels: \\An Embedding Perspective} 

\titlerunning{Expanding the Family of Grassmannian Kernels}

\authorrunning{Harandi \etal}

\author
  {
  Mehrtash~T.~Harandi
  \and
  Mathieu~Salzmann
  \and
  Sadeep~Jayasumana
  \and
  \\Richard~Hartley
  \and
  Hongdong~Li  
  }

\institute
  {
  Australian National University, Canberra, ACT 0200, Australia\\
  ~NICTA\thanks
    {
    NICTA is funded by the Australian Government as represented by the Department of
    Broadband, Communications and the Digital Economy, as well as by the Australian
    Research Council through the ICT Centre of Excellence program.
    }%
    , Locked Bag 8001, Canberra, ACT 2601, Australia%
  }

\maketitle

\begin{abstract}
Modeling videos and image-sets as linear subspaces has proven beneficial for many visual recognition tasks. However, it also incurs challenges arising from the fact that linear subspaces do not obey Euclidean geometry, but lie on a special type of Riemannian manifolds known as Grassmannian. To leverage the techniques developed for Euclidean spaces (\eg, support vector machines) with subspaces, several recent studies have proposed to embed the Grassmannian into a Hilbert space by making use of a positive definite kernel. Unfortunately, only two Grassmannian kernels are known, none of which -as we will show- is \emph{universal}, which limits their ability to approximate a target function arbitrarily well. Here, we introduce several positive definite Grassmannian kernels, including universal ones, and demonstrate their superiority over previously-known kernels in various tasks, such as classification, clustering, sparse coding and hashing.
\end{abstract}

\section{Introduction}
\label{sec:intro}

This paper introduces a set of positive definite kernels  to embed Grassmannians (\ie, manifolds of linear subspaces that have a nonlinear
Riemannian structure) into Hilbert spaces, which have a more familiar Euclidean structure. 
Nowadays, linear subspaces are a core representation of many visual recognition techniques. 
For example, several state-of-the-art video, or image-set, matching methods model the visual data as 
subspaces~\cite{Hamm:ICML:2008,Harandi:CVPR:2011,Turaga:PAMI:2011,Vemula:CVPR:2013,Jayasumana_2014_CVPR}.
Linear subspaces have also proven a powerful representation for many other computer vision applications, such as chromatic noise filtering~\cite{Subbarao:IJCV:2009} and
domain adaptation~\cite{Gopalan:PAMI:2013}.

Despite their success, linear subspaces suffer from the drawback that they cannot be analyzed using Euclidean geometry. Indeed, subspaces lie on a special type of Riemannian manifolds, the Grassmann manifold, which has a nonlinear structure. As a consequence, popular techniques developed for Euclidean spaces do not apply. Recently, this problem has been addressed by embedding the Grassmannian into a Hilbert space. This can be achieved either by tangent space approximation of the manifold, or by exploiting a positive definite kernel function to embed the manifold into a reproducing kernel Hilbert space (RKHS). In either case, any existing Euclidean technique can then be applied to the embedded data, since Hilbert spaces obey Euclidean geometry. Recent studies, however, report superior results with RKHS embedding over flattening the manifold using its tangent spaces~\cite{Hamm:ICML:2008,Vemula:CVPR:2013,Sadeep:CVPR:2013}. Intuitively, this can be attributed to the fact that a tangent space is a first order approximation to the true geometry of the manifold, whereas, being higher-dimensional, an RKHS has the capacity of better capturing the nonlinearity of the manifold.

While RKHS embeddings therefore seem preferable, their applicability is limited by the fact that only very few positive definite Grassmannian kernels are known. Indeed, in the literature, only two kernels have been introduced to embed Grassmannians into RKHS: the Binet-Cauchy kernel~\cite{Wolf:JMLR:2003} and the projection kernel~\cite{Hamm:ICML:2008}. The former is a homogeneous second order polynomial kernel, while the latter is a linear kernel. As simple (low-order) polynomial kernels, they are limited in their ability to closely approximate arbitrary functions. In contrast, universal kernels provide much better generalization power~\cite{Steinwart:Book,Micchelli:JMLR:2006}.

In this paper, we introduce a set of new positive definite Grassmannian kernels, which, among others, includes universal Grassmannian kernels. To this end, we start from the perspective of the two embeddings from which the Binet-Cauchy and the projection kernels are derived: the Pl\"ucker embedding and the projection embedding. These two embeddings yield two distance functions. We then exploit the properties of these distances, in conjunction with several theorems analyzing the positive definiteness of kernels, to derive the ten new Grassmannian kernels summarized in Table~\ref{tab:grassmannian_kernel}. 

Our experimental evaluation demonstrates the benefits of our Grassmannian kernels for classification, clustering, sparse coding and hashing. Our results show that our kernels outperform the Binet-Cauchy and projection ones for gender and gesture recognition, pose categorization and mouse behavior analysis. 

\begin{table}[!tb] 
{\centering 
  \caption  {    \small
    The proposed Grassmannian kernels and their properties.
    }
  \label{tab:grassmannian_kernel}
  {\small
  \begin{tabular}{m{1.75cm} m{6.5cm} m{1.5cm} m{2cm}}
    \toprule
    {\bf Kernel}   & {\bf Equation}  &{\bf Cond.} &{\bf Properties}\\
    \toprule
    \multirow{2}{*}{Polynomial}
    & \(\displaystyle k_{p,bc}(\Mat{X},\Mat{Y}) = \Big(\beta + \big|\det\big(\Mat{X}^T\Mat{Y} \big)\big| \Big)^\alpha\)
    & \(\displaystyle \beta > 0 \)
    & \emph{pd}\\
    & \(\displaystyle k_{p,p}(\Mat{X},\Mat{Y}) = \Big(\beta + \big\|\Mat{X}^T\Mat{Y} \big\|_F^2 \Big)^\alpha\)
    & \(\displaystyle \beta > 0 \)
    & \emph{pd}\\[3ex]
    \hline
    \multirow{2}{*}{RBF}      
    &\(\displaystyle k_{r,bc}(\Mat{X},\Mat{Y}) = \exp \Big(\beta \big|\det \big(\Mat{X}^T\Mat{Y} \big)\big| \Big)\)
    & \(\displaystyle \beta > 0 \)
    &\emph{pd}, universal\\[1ex]
    &\(\displaystyle k_{r,p}(\Mat{X},\Mat{Y}) = \exp \Big(\beta \big\|\Mat{X}^T\Mat{Y} \big\|_F^2 \Big)\)
    & \(\displaystyle \beta > 0 \)
    & \emph{pd}, universal\\[3ex]
    \hline
    \multirow{2}{*}{Laplace}      
    &\(\displaystyle k_{l,bc}(\Mat{X},\Mat{Y}) = \exp \bigg(-\beta \sqrt{1 - \big|\det (\Mat{X}^T\Mat{Y})\big|} \bigg)\)
    & \(\displaystyle \beta > 0 \)
    &\emph{pd}, universal\\[3ex]
    &\(\displaystyle k_{l,p}(\Mat{X},\Mat{Y}) = \exp \bigg(-\beta \sqrt{p - \big\|\Mat{X}^T\Mat{Y} \big\|_F^2} \bigg) \)
    & \(\displaystyle \beta > 0 \)
    & \emph{pd}, universal\\[3ex]
    \hline
    \multirow{2}{*}{Binomial}   
    & \(\displaystyle 	k_{bi,bc}(\Mat{X},\Mat{Y}) = \Big(\beta - \big|\det \big(\Mat{X}^T\Mat{Y} \big)\big|\Big)^{-\alpha} \)
    & \(\displaystyle \beta > 1 \)
    & \emph{pd}, universal\\[3ex]
    & \(\displaystyle k_{bi,p}(\Mat{X},\Mat{Y}) = \Big(\beta - \big\|\Mat{X}^T\Mat{Y} \big\|_F^2\Big)^{-\alpha} \)
    & \(\displaystyle \beta > p \)
    & \emph{pd}, universal\\
    \hline
    \multirow{2}{*}{Logarithm}   
    & \(\displaystyle k_{log,bc}(\Mat{X},\Mat{Y}) = -\log\bigg( 2 - \big|\det\big(\Mat{X}^T\Mat{Y} \big)\big| \bigg) \)
    & \(\displaystyle - \)
    & \emph{cpd} \\[3ex] 
    & \(\displaystyle k_{log,p}(\Mat{X},\Mat{Y}) = -\log\bigg( p + 1 - \big\|\Mat{X}^T\Mat{Y} \big\|_F^2 \bigg)\)
    & \(\displaystyle - \)
    & \emph{cpd} \\
    \bottomrule
  \end{tabular}
  }
}  
\end{table}

\section{Background Theory}
\label{sec:grassmannian}
%
In this section, we first review some notions of geometry of Grassmannians and then briefly discuss existing positive definite kernels and their properties.
Throughout the paper, we use bold capital letters to denote matrices (\eg, $\Mat{X}$) and bold lower-case letters to denote column vectors (\eg, $\Vec{x}$).
$\mathbf{I}_p$ is the $p \times p$ identity matrix.
$\Vert \Mat{X} \Vert_F = \sqrt{\tr \big(\Mat{X}^T\Mat{X}\big)}$ indicates the Frobenius norm,
with $\tr(\cdot)$ the matrix trace.

\subsection{Grassmannian Geometry}

The space of $p$-dimensional linear subspaces of $\mathbb{R}^d$ for $0<p<d$ is not a Euclidean space, but a Riemannian manifold known as the Grassmannian $\GRASS{d}{p}$~\cite{Absil:2008}.
We note that in the special case of $p = 1$, the Grassmann manifold becomes the projective space $\mathbb{P}^{d - 1}$, which consists of all lines passing through the origin.
A point on the Grassmann manifold~$\GRASS{p}{d}$ may be specified by an arbitrary $d \times p$ matrix with orthogonal columns, \ie, $\Mat{X} \in \GRASS{d}{p} \Rightarrow \Mat{X}^T\Mat{X} = \mathbf{I}_p$%
\footnote{A point on the Grassmannian $\GRASS{p}{d}$ is a subspace spanned by the columns of a $d \times p$ full rank matrix and should therefore be denoted by $\operatorname{span}(\Mat{X})$.
With a slight abuse of notation, here we call $\Mat{X}$ a Grassmannian point whenever it represents a basis for a subspace.}.

On a Riemannian manifold, points are connected via smooth curves. The distance between two points is
defined as the length of the shortest curve connecting them on the manifold. The shortest curve and
its length are called geodesic and geodesic distance, respectively. For the Grassmannian, the geodesic distance between 
two points $\Mat{X}$ and $\Mat{Y}$ is given by
\begin{equation}
	\delta_g(\Mat{X},\Mat{Y}) = \|\Theta\|_2\;,
	\label{eqn:geodesic_dist}
\end{equation}
where $\Theta$ is the vector of principal angles between $\Mat{X}$ and $\Mat{Y}$.

\begin{definition}[Principal Angles]
Let $\Mat{X}$ and $\Mat{Y}$ be two matrices of size $d \times p$ with orthonormal columns. The principal angles
$0 \leq \theta_1 $ $\leq \theta_2 \leq $ $\cdots $ $ \leq \theta_p \leq \pi/2$ between two subspaces
$\operatorname{span}(\Mat{X})$ and $\operatorname{span}(\Mat{Y})$ are defined recursively by
\end{definition}
\begin{eqnarray}
  &\cos(\theta_i)
  =
  \underset{\Vec{u}_i \in \operatorname{span}(\Mat{X})}{\max}\;
  \underset{\Vec{v}_i \in \operatorname{span}(\Mat{Y})}{\max}\;
  \Vec{u}_i^T \Vec{v}_i  \\
  \text{s.t.}
  &\|\Vec{u}_i\|_2 \mbox{~=~} \|\Vec{v}_i\|_2 \mbox{~=~} 1  \nonumber\\
  &\Vec{u}_i^T \Vec{u}_j \mbox{~=~} 0,\; j=1,2,\cdots,i-1     \nonumber\\
  &\Vec{v}_i^T \Vec{v}_j \mbox{~=~} 0,\; j=1,2,\cdots,i-1     \nonumber
  \label{eqn:Principal_Angle}
\end{eqnarray}%
In other words, the first principal angle $\theta_1$ is the smallest
angle between any two unit vectors in the first and the second subspaces. 
The cosines of the principal angles correspond to the singular values of \mbox{$\Mat{X}^T \Mat{Y}$}~\cite{Absil:2008}.
In addition to the geodesic distance, several other metrics can be employed to measure the similarity between Grassmannian points~\cite{Hamm:ICML:2008}. In Section~\ref{sec:embedding}, we will discuss two other metrics on the Grassmannian.

\subsection{Positive Definite Kernels and Grassmannians}

As mentioned earlier, a popular way to analyze problems defined on a Grassmannian is to embed the manifold into a Hilbert space
using a valid Grassmannian kernel. Let us now formally define Grassmannian kernels:

\begin{definition}[Real-valued Positive Definite Kernels]
Let $\mathcal{X}$ be a nonempty set. A symmetric function $k: \mathcal{X} \times \mathcal{X} \to \mathbb{R}$ is a positive definite (\emph{\textbf{pd}})
kernel on $\mathcal{X}$ if and only if
$\sum_{i,j=1}^nc_ic_jk(x_i,x_j) \geq 0$
for any $n \in \mathbb{N}$, $x_i \in \mathcal{X}$ and $c_i \in \mathbb{R}$. 
\end{definition}

\begin{definition}[Grassmannian Kernel]
A function $k:\GRASS{p}{d} \times \GRASS{p}{d} \to \mathbb{R}$ is a Grassmannian kernel if it is well-defined and \emph{pd}. 
In our context, a function is well-defined if it is invariant to the choice of basis, \ie, 
$k(\Mat{X}\Mat{R}_1,\Mat{Y}\Mat{R}_2) = k(\Mat{X},\Mat{Y})$, for all $\Mat{X}, \Mat{Y} \in \GRASS{d}{p}$ and
$\Mat{R}_1,\Mat{R}_2 \in \mathrm{SO}(p)$, where $\mathrm{SO}(p)$ denotes the special orthogonal group.
\end{definition}

The most widely used kernel is arguably the Gaussian or radial basis function (RBF) kernel.
It is therefore tempting to define a Radial Basis Grassmannian kernel by replacing the Euclidean distance with the geodesic distance.
Unfortunately, although symmetric and well-defined, the function $\exp(-\beta \delta_g^2(\cdot,\cdot))$ is not \emph{pd}.
This can be verified by a counter-example using the following points on $\GRASS{3}{2}$\footnote%
{Note that we rounded each value to its 4 most significant digits.}:%
{\scriptsize
\begin{align*}
\Mat{X}_1 = 
\begin{bmatrix}
1 &0\\
0 &1\\
0 &0
\end{bmatrix},~
\Mat{X}_2 = 
\begin{bmatrix}
-0.0996 &-0.3085\\
-0.4967 &-0.8084\\
-0.8622 &0.5014
\end{bmatrix},~
\Mat{X}_3 = 
\begin{bmatrix}
-0.9868 &0.1259\\
-0.1221 &-0.9916\\
-0.1065 &-0.0293
\end{bmatrix},~
\Mat{X}_4 = 
\begin{bmatrix}
0.1736 &0.0835\\
0.7116 &0.6782\\
0.6808 &-0.7301
\end{bmatrix}.
\end{align*}
}
The function $\exp(- \delta_g^2(\cdot,\cdot))$ for these points has a negative eigenvalue of $-0.0038$.

Nevertheless, two Grassmannian kernels, \ie, the Binet-Cauchy kernel~\cite{Wolf:JMLR:2003} and the projection kernel~\cite{Hamm:ICML:2008}, have been proposed to embed Grassmann manifolds into RKHS. The Binet-Cauchy and projection kernels are defined as
\begin{align}
	\label{eqn:bc_kernel}
	k_{bc}^2(\Mat{X},\Mat{Y}) &= {\det}\big(\Mat{X}^T\Mat{Y}\Mat{Y}^T\Mat{X}\big)\;,\\
	k_p(\Mat{X},\Mat{Y}) &= \big\|\Mat{X}^T\Mat{Y}\big\|_F^2\;.
	\label{eqn:proj_kernel}
\end{align}

\paragraph{\bf Property 1 (Relation to Principal Angles).}
Both $k_p$ and $k_{bc}$ are closely related to the principal angles between two subspaces.
Let $\theta_i$ be the $i^{th}$ principal angle between $\Mat{X},\Mat{Y} \in \GRASS{p}{d}$,
\ie, by SVD, $\Mat{X}^T\Mat{Y} = \Mat{U}\Gamma\Mat{V}^T$, with $\Gamma$ a diagonal matrix with elements $\cos\theta_i$. Then
\begin{align*}
	k_p(\Mat{X},\Mat{Y}) &= \big\|\Mat{X}^T\Mat{Y}\big\|_F^2 
	= \tr\left(\Mat{U}\Gamma\Mat{V}^T\Mat{V}\Gamma\Mat{U}^T\right) = \tr\left(\Gamma^2\right)
	=\sum_{i = 1}^p\limits \cos^2(\theta_i)\;.
\end{align*}%
Similarly, one can show that $k_{bc}^2(\Mat{X},\Mat{Y}) = \prod_{i = 1}^p\limits \cos^2(\theta_i)$.

\section{Embedding Grassmannians to Hilbert Spaces}
\label{sec:embedding}

While $k_p$ and $k_{bc}^2$ have been successfully employed to transform problems on Grassmannians to 
 Hilbert spaces~\cite{Hamm:ICML:2008,Harandi:CVPR:2011,Vemula:CVPR:2013}, 
the resulting Hilbert spaces themselves have received comparatively little attention. In this section, we aim to bridge this gap and study these two spaces, which can be explicitly computed. To this end, we discuss the two embeddings that define these Hilbert spaces, namely the Pl\"{u}cker embedding and the projection embedding. These embeddings, and their respective properties, will in turn help us devise our set of new Grassmannian kernels. 

\subsection{Pl\"{u}cker Embedding}

To study the Pl\"{u}cker embedding, we first need to review some concepts of exterior algebra.
\begin{definition}[Alternating Multilinear Map]
Let $\Mat{V}$ and $\Mat{W}$ be two vector spaces. A map $g:\underbrace{\Mat{V} \times \cdots \times \Mat{V}}_{k~copies} \to \Mat{W}$
is multilinear if it is linear in each slot, that is if
\begin{equation*}
g(\Vec{v}_1,\cdots,\lambda\Vec{v}_i + \lambda'\Vec{v}'_i,\cdots,\Vec{v}_k) = 
\lambda g(\Vec{v}_1,\cdots,\Vec{v}_i,\cdots,\Vec{v}_k) + 
\lambda' g(\Vec{v}_1,\cdots,\Vec{v}'_i,\cdots,\Vec{v}_k)\;.
\end{equation*} 
Furthermore, the map $g$ is alternating if, whenever two of the inputs to $g$ are the same vector, the output is 0.
That is, if $g(\cdots,\Vec{v},\cdots,\Vec{v},\cdots) = 0,~\forall \Vec{v}$.
\end{definition}

\begin{definition}[$k^{th}$ Exterior Product]
Let $\Mat{V}$ be a vector space. The $k^{th}$ exterior product of $\Mat{V}$, denoted by $\bigwedge^k\Mat{V}$ 
is a vector space, equipped with an alternating  multilinear map 
$g:\underbrace{\Mat{V} \times \cdots \times \Mat{V}}_{k~copies} \to \bigwedge^k\Mat{V}$ of the form 
$g(\Vec{v}_1,\cdots,\Vec{v}_k) = \Vec{v}_1 \wedge \cdots \wedge \Vec{v}_k$,
with $\wedge$ the wedge product.
\end{definition}
The wedge product is  supercommutative
and can be thought of as a generalization of the cross product in $\mathbb{R}^3$ to an arbitrary dimension.
Importantly, note that the $k^{th}$ exterior product $\bigwedge^k\Mat{V}$ is a vector space, that is
\begin{equation*}
\bigwedge^k\Mat{V} = \rm{span}\left(\{ \Vec{v}_1 \wedge \Vec{v}_2 \wedge \cdots \wedge \Vec{v}_k \}\right),\; \forall \Vec{v}_i \in \Mat{V}\;.
\end{equation*}
The Grassmannian $\GRASS{d}{p}$ can be embedded into the projective space $\mathbb{P}(\bigwedge^p\mathbb{R}^{d})$ as follows.
Let $\Mat{X} $ be a point on $\GRASS{p}{d}$ described by the basis $\{\Vec{x}_1 , \Vec{x}_2 , \cdots , \Vec{x}_p\}$, 
\ie, $\Mat{X} = \mathrm{span}\left(\{\Vec{x}_1 , \Vec{x}_2 , \cdots , \Vec{x}_p\}\right)$.
The Pl\"{u}cker map of $\Mat{X}$ is given by:
\begin{definition}[Pl\"{u}cker Embedding] 
The Pl\"{u}cker embedding $P:\GRASS{p}{d} \to \mathbb{P}(\bigwedge^p \mathbb{R}^d)$ is defined as
\begin{equation}
	P(\Mat{X}) = [\Vec{x}_1 \wedge \Vec{x}_2 \wedge \cdots \wedge \Vec{x}_p]\,,
	\label{eqn:plucker_emb}
\end{equation}
where $\Mat{X}$ is the subspace spanned by $\{\Vec{x}_1 , \Vec{x}_2 , \cdots , \Vec{x}_p\}$.
\end{definition}

\begin{example} 
Consider the space of two-dimensional planes in $\mathbb{R}^4$, \ie, $\GRASS{2}{4}$. In this space, an arbitrary subspace 
is described by its basis ${\mathbf B} = [\Vec{w}_1 | \Vec{w}_2]$. Let $\Vec{e}_i$ be the unit vector along the $i^{th}$ axis.
We can write $\Vec{w}_j = \sum_{i=1}^4 a_{j,i}\Vec{e}_i$. Then
\begin{align*}
P({\mathbf B}) &= \Big( \sum_{i=1}^4 a_{1,i}\Vec{e}_i \Big) \wedge \Big( \sum_{i=1}^4 a_{2,j}\Vec{e}_j \Big) \\
               &= (a_{1,1}a_{2,2} - a_{1,2}a_{2,1}) (\Vec{e}_1 \wedge \Vec{e}_2) +
               (a_{1,1}a_{2,3} - a_{1,3}a_{2,1}) (\Vec{e}_1 \wedge \Vec{e}_3) \\
               &+(a_{1,1}a_{2,4} - a_{1,4}a_{2,1}) (\Vec{e}_1 \wedge \Vec{e}_4) +
               (a_{1,2}a_{2,3} - a_{1,3}a_{2,2}) (\Vec{e}_2 \wedge \Vec{e}_3) \\
               &+(a_{1,2}a_{2,4} - a_{1,4}a_{2,2}) (\Vec{e}_2 \wedge \Vec{e}_4) +
               (a_{1,3}a_{2,4} - a_{1,4}a_{2,3}) (\Vec{e}_3 \wedge \Vec{e}_4)\;.                 
\end{align*}
Hence, the Pl\"{u}cker embedding of $\GRASS{2}{4}$ is a 6-dimensional space spanned by 
$\{\Vec{e}_1 \wedge \Vec{e}_2$, $\Vec{e}_1 \wedge \Vec{e}_3$, $\cdots$, $\Vec{e}_3 \wedge \Vec{e}_4\}$.
A closer look at the coordinates of the embedded subspace reveals that they are indeed the minors of all possible $2 \times 2$
submatrices of ${\mathbf B}$. This can be shown to hold for any $d$ and $p$. 
\end{example}

\begin{proposition}
The Pl\"{u}cker coordinates of $\Mat{X} \in \GRASS{d}{p}$ are the $p \times p$ minors of the matrix $\Mat{X}$ 
obtained by taking $p$ rows out of the $d$ possible ones.
\end{proposition}

\begin{remark}
The space induced by the Pl\"{u}cker map of $\GRASS{p}{d}$ is $\begin{pmatrix} p\\d \end{pmatrix}$-dimensional.
\end{remark}

To be able to exploit the Pl\"{u}cker embedding to design new kernels, we need  to define an inner product over $\mathbb{P}(\bigwedge^p\mathbb{R}^{d})$. Importantly, to be meaningful, this inner product needs to be
invariant to the specific realization of a point on $\GRASS{p}{d}$ (recall that, \eg, swapping two columns of a specific realization $\Mat{X} \in \GRASS{d}{p}$ still corresponds to the same point on $\GRASS{d}{p}$). 
Furthermore, we would also like this inner product to be efficient to evaluate, thus avoiding the need to explicitly compute the high-dimensional embedding. Note in particular that, for vision applications, the dimensionality of $\mathbb{P}(\bigwedge^p\mathbb{R}^{d})$ becomes overwhelming and hence explicitly computing the embedding is impractical. To achieve these goals, we rely on the following definition and theorem:
\begin{definition}[Compound Matrices] 
Given a $d \times p$ matrix $\Mat{A}$, the matrix whose elements are the
minors of $\Mat{A}$ of order $q$ arranged in a lexicographic
order is called the $q^{th}$ compound of $\Mat{A}$, and
is denoted by $\Mat{C}_q(\Mat{A})$.
\end{definition}
\begin{theorem}[Binet-Cauchy Theorem]
Let $\Mat{A}$ and $\Mat{B}$ be two rectangular matrices of size $d \times p_1$ and $d \times p_2$,
respectively. Then,
$\Mat{C}_q(\Mat{A}^T\Mat{B}) = \Mat{C}_q(\Mat{A})^T\Mat{C}_q(\Mat{B})$.
\end{theorem}
\noindent
Therefore, for $\Mat{X}, \Mat{Y} \in \mathbb{R}^{d \times p}$, 
we have $\tr\Big(\Mat{C}_p(\Mat{X})^T\Mat{C}_p(\Mat{Y})\Big) 
= \tr\Big(\Mat{C}_p(\Mat{X}^T\Mat{Y})\Big) = \det \big( \Mat{X}^T\Mat{Y} \big)$. 

Since, for $\Mat{X} \in \GRASS{p}{d}$, $\Mat{C}_p(\Mat{X})$ stores all $p \times p$ minors and hence conveys the Pl\"{u}cker coordinates of $\Mat{X}$, this would suggest defining the inner product for the Pl\"{u}cker embedding as $\det ( \Mat{X}^T\Mat{Y} \big)$. This is indeed what was proposed in~\cite{Hamm:ICML:2008,Wolf:JMLR:2003} where $\det (\cdot)$ was used as a linear kernel. However, while $\det ( \Mat{X}^T\Mat{Y} \big)$ is invariant to the action of $\mathrm{SO}(p)$, it is not invariant to the specific realization of a subspace. This can be simply verified by permuting the columns of $\Mat{X}$, which does not change the subspace, but may change the sign of $\det (\cdot )$. Note that this sign issue was also observed by Wolf~\etal~\cite{Wolf:JMLR:2003}. However, this problem was circumvented by considering the second-order polynomial kernel $k^2_{bc}$.

In contrast, here, we focus on designing a valid inner product that satisfies this invariance condition. To this end, we define the inner product in $\mathbb{P}(\bigwedge^p\mathbb{R}^{d})$ as
$\langle\Mat{X},\Mat{Y}\rangle_{P} = |P(\Mat{X})^TP(\Mat{Y})| = \big|\det \big(\Mat{X}^T\Mat{Y}\big)\big|$. This inner product induces the distance
\begin{equation}
	\delta_{bc}^2(\Mat{X},\Mat{Y}) = \|P(\Mat{X}) - P(\Mat{Y})\|^2 = 2 - 2\big|\det\big(\Mat{X}^T\Mat{Y}\big)\big|\;.
	\label{eqn:plucker_dist}
\end{equation}
Clearly, if $\{\theta_i\}_{i=1}^p$ is the set of principal angles between two Grassmannian points $\Mat{X}$ and $\Mat{Y}$, then  
$\langle\Mat{X},\Mat{Y}\rangle_{P} = \prod\nolimits_{i = 1}^p \cos(\theta_i)$, which is invariant to the specific realization of a subspace since $0 \leq \theta_i \leq \pi/2$.

In the following, we show that the Pl\"{u}cker embedding has the nice property of being closely related to the true geometry of the corresponding Grassmannian:
\begin{theorem}[Curve Length Equivalence]
The length of any given curve is the same under $\delta_{bc}$ and $\delta_g$ up to a scale of $\sqrt{2}$.
\end{theorem}
\begin{proof}
Given in appendix.\qed
\end{proof}

\subsection{Projection Embedding}

We now turn to the case of the projection embedding. Note that this embedding has been better studied than the Pl\"ucker one~\cite{Helmke-Grassmann}.

\begin{definition}[Projection Embedding]
The projection embedding $\Pi: \GRASS{p}{d} \to \rm{Sym}(d)$ is defined as
\begin{equation}
	\Pi(\Mat{X}) = \Mat{X}\Mat{X}^T\,.
	\label{eqn:projection_emb}
\end{equation}
\end{definition}

The projection embedding $\Pi(\cdot)$ is a diffeomorphism from a Grassmann manifold onto the idempotent symmetric matrices
of rank $p$, \ie, it is a one-to-one, continuous, differentiable mapping with a continuous, differentiable inverse~\cite{Chikuse:2003}.
The space induced by this embedding is a smooth, compact submanifold of $\rm{Sym}(d)$ of dimension $d(d - p)$.
Since $\Pi(\Mat{X})$ is a symmetric $d \times d$ matrix, a natural choice of inner product is $\langle \Mat{X},\Mat{Y} \rangle_\Pi = \tr \big(\Pi(\Mat{X})^T\Pi(\Mat{Y})\big) = \big\|\Mat{X}\Mat{Y}\big\|_F^2$. This inner product can be shown to be invariant to the specific realization of a subspace, and induces the distance
\begin{equation*}
	\delta_{p}^2(\Mat{X},\Mat{Y}) = \big\|\Pi(\Mat{X}) - \Pi(\Mat{Y}) \big\|_F^2 = 2p - 2\big\|\Mat{X}^T\Mat{Y} \big\|_F^2\;.
\end{equation*}  

Due to space limitation, we do not discuss the properties of the projection embedding, such as isometry~\cite{Chikuse:2003} and length of curves~\cite{Harandi:ICCV:2013}. We refer the reader to~\cite{Helmke-Grassmann} for a more thorough discussion of the projection embedding.

\section{Grassmannian Kernels}
\label{sec:sec_kernels}

From the discussion in Section~\ref{sec:embedding}, $k^2_{bc}$ and $k_p$, defined
in Eq.~\ref{eqn:proj_kernel} and Eq.~\ref{eqn:bc_kernel}, can be seen to correspond to a homogeneous second order polynomial kernel in the space induced by the Pl\"ucker embedding
 and to a linear kernel in the space induced by the projection embedding, respectively. In this section, we show that the inner products that we defined in Section~\ref{sec:embedding} for the Pl\"{u}cker and projection embeddings can actually be exploited to derive many new Grassmannian kernels, including universal kernels and conditionally positive definite kernels. In the following, we denote by $k_{\cdot,bc}$ and $k_{\cdot,p}$ kernels derived from the Pl\"ucker embedding (Binet-Cauchy kernels) and from the projection embedding, respectively.

\subsection{Polynomial Kernels}

Given an inner product, which itself defines a valid linear kernel, the most straightforward way to create new kernels is to consider higher degree polynomials. Such polynomial kernels are known to be {\it pd}.
Therefore, we can readily define polynomial kernels on the Grassmannian as
\begin{align}
	k_{p,bc}(\Mat{X},\Mat{Y}) &= \Big(\beta + \big|\det\big(\Mat{X}^T\Mat{Y} \big)\big| \Big)^\alpha~,~~\beta > 0\;, \\
	k_{p,p}(\Mat{X},\Mat{Y}) &= \Big(\beta + \big\|\Mat{X}^T\Mat{Y} \big\|_F^2 \Big)^\alpha~,~~\beta > 0\;.
	\label{eqn:poly_grassmann}
\end{align}
Note that the kernel used in~\cite{Wolf:JMLR:2003} is indeed the homogeneous second order $k_{p,bc}$ with $\alpha = 2$ and $\beta = 0$. 

\subsection{Universal Grassmannian Kernels}
\label{subsec:universal_kernel}
Although often used in practice, polynomial kernels are known not to be universal~\cite{Steinwart:Book}. This can have a crucial impact on their representation power for a specific task. Indeed, from the \emph{Representer Theorem}~\cite{Scholkopf:COLT:2001}, we have that, for a given set of training data $\{x_j\},\;j \in \mathbb{N}_n$, $\mathbb{N}_n = \{1,2,\cdots,n\}$ and a {\it pd} kernel $k$, the function learned by any algorithm can be expressed as
\begin{equation}
	\hat{f}(x_*) = \sum_{j \in \mathbb{N}_n}c_jk(x_*,x_j)\;.
	\label{eqn:representer}
\end{equation}
Importantly, only {\it universal kernels} have the property of being able to approximate any target function $f_t$ arbitrarily well given sufficiently many training samples. Therefore, $k_p$ and $k^2_{bc}$ may not generalize sufficiently well for certain problems.
In the following, we develop several universal Grassmannian kernels. To this end, we make use of negative definite kernels and of their 
relation to \emph{pd} ones. Let us first formally define negative definite kernels.

\begin{definition}[Real-valued Negative Definite Kernels]
Let $\mathcal{X}$ be a nonempty set. A symmetric function $\psi: \mathcal{X} \times \mathcal{X} \to \mathbb{R}$ is a negative definite (\emph{\textbf{nd}})
kernel on $\mathcal{X}$ if and only if $\sum_{i,j=1}^nc_ic_jk(x_i,x_j) \leq 0$ for any $n \in \mathbb{N}$, $x_i \in \mathcal{X}$ and $c_i \in \mathbb{R}$ with $\sum_{i=1}^n c_i = 0$.
\end{definition}
Note that, in contrast to positive definite kernels, an additional constraint of the form $\sum c_i = 0$ is required in the negative definite case. 

The most important example of \emph{nd} kernels is the distance function defined on a Hilbert space. More specifically:

\begin{theorem}[\cite{Sadeep:CVPR:2013}] \label{thm:neg_def_dist}
Let $\mathcal{X}$ be a nonempty set, $\mathcal{H}$ be an inner product space, and $\psi : \mathcal{X} \to \mathcal{H}$ be a function. Then
$f : (\mathcal{X} \times \mathcal{X}) \to \mathbb{R}$ defined by $f(x_i, x_j) = {\| \psi(x_i) - \psi(x_j) \|}^2_{\mathcal{H}}$ is negative definite.
\end{theorem}

Therefore, being distances in Hilbert spaces, both $\delta_{bc}^2$ and $\delta_p^2$ are \emph{nd} kernels. We now state an important theorem which establishes
the relation between \emph{pd} and \emph{nd} kernels.

\begin{theorem}[Theorem 2.3 in Chapter 3 of~\cite{Berg:1984}] \label{thm:laplace_kernel}
Let $\mu$ be a probability measure on the half line $\mathbb{R}_{+}$ and  
$0 < \int_{0}^{\infty}t\mathrm{d}\mu(t) < \infty$. Let $\mathcal{L}_{\mu}$ be the Laplace transform of $\mu$, \ie, 
$\mathcal{L}_\mu(s) = \int_{0}^{\infty}e^{-ts}\mathrm{d}\mu(t),\; s \in \mathbb{C}_{+}$. Then, $\mathcal{L}_\mu(\beta f)$
is positive definite for all $\beta > 0$  if and only if $f:\mathcal{X} \times \mathcal{X} \to \mathbb{R}_{+}$  is negative definite.
\end{theorem}

The problem of designing a \emph{pd} kernel on the Grassmannian can now be cast as that of finding an appropriate probability measure $\mu$.
Below, we show that this lets us reformulate popular kernels in Euclidean space as Grassmannian kernels.

\subsubsection{RBF Kernels.}
Grassmannian RBF kernels can be obtained by choosing $\mu(t) = \delta(t-1)$ in Theorem~\ref{thm:laplace_kernel}, where $\delta(t)$ is the Dirac delta function. This choice yields the Grassmannian RBF kernels (after discarding scalar constants)
\begin{align}
	k_{r,bc}(\Mat{X},\Mat{Y}) &= \exp \Big(\beta \big|\det \big(\Mat{X}^T\Mat{Y} \big)\big| \Big)~,~~\beta > 0\;, \\
	k_{r,p}(\Mat{X},\Mat{Y}) &= \exp \Big(\beta \big\|\Mat{X}^T\Mat{Y} \big\|_F^2 \Big)~,~~\beta > 0\;.
	\label{eqn:rbf_grassmann}
\end{align}
Note that the RBF kernel obtained from the projection embedding, \ie $k_{r,p}$, was also used by Vemulapalli 
\etal~\cite{Vemula:CVPR:2013}. However, the positive definiteness of this kernel was neither proven nor discussed.

\subsubsection{Laplace Kernels.}

The Laplace kernel is another widely used Euclidean kernel, defined as $k(\Vec{x},\Vec{y}) = \exp(-\beta\|\Vec{x} - \Vec{y}\|)$.
To obtain Laplace kernels on the Grassmannian, we make use of the following theorem for \emph{nd} kernels.

\begin{theorem}[Corollary 2.10 in Chapter 3 of~\cite{Berg:1984}] \label{thm:power_nd}
If $\psi:\mathcal{X} \times \mathcal{X} \to \mathbb{R}$ is negative definite and satisfies $\psi(\Vec{x},\Vec{x}) \geqq 0$
then so is $\psi^\alpha$ for $0 < \alpha < 1$.
\end{theorem}

As a result, both $\delta_p(\cdot,\cdot)$ and $\delta_{bc}(\cdot,\cdot)$ are \emph{nd} by choosing $\alpha = 1/2$
in Theorem~\ref{thm:power_nd}. 
By employing either $\delta_p(\cdot,\cdot)$ or $\delta_{bc}(\cdot,\cdot)$ along with $\mu(t) = \delta(t-1)$ in Theorem~\ref{thm:laplace_kernel}, we obtain the Grassmannian Laplace kernels
\begin{align}
	k_{l,bc}(\Mat{X},\Mat{Y}) &= \exp \bigg(-\beta \sqrt{1 - \big|\det \big(\Mat{X}^T\Mat{Y} \big)\big|} \bigg)~,~~\beta > 0\;,
	\label{eqn:laplace_grassmann_bc}\\
	k_{l,p}(\Mat{X},\Mat{Y}) &= \exp \bigg(-\beta \sqrt{p - \big\|\Mat{X}^T\Mat{Y} \big\|_F^2} \bigg)~,~~\beta > 0\;.
	\label{eqn:laplace_grassmann_proj}
\end{align}

As shown in~\cite{Steinwart:Book}, the RBF and Laplace kernels are universal for $\mathbb{R}^d, d>0$. Since the Pl\"{u}cker and projection embeddings map to Euclidean spaces, this property clearly extends to the Grassmannian RBF and Laplace kernels.

\subsubsection{Binomial Kernels.}
By choosing $\mu(t) = \exp(-\beta_0 t)u(t)$, where $u(t)$ is the unit (or Heaviside) step function, \ie, $u(t) = \int_{-\infty}^t\delta(x)\mathrm{d}x$,
we obtain the Grassmannian binomial kernels
\begin{align}	 
	k_{b,bc}(\Mat{X},\Mat{Y}) &= \frac{1}{\beta - \big|\det \big(\Mat{X}^T\Mat{Y} \big)\big|}~,~~\beta > 1\;,
	\label{eqn:binom_grassmann_bc}\\
	k_{b,p}(\Mat{X},\Mat{Y}) &= \frac{1}{\beta - \big\|\Mat{X}^T\Mat{Y} \big\|_F^2}~,~~\beta > p\;. 
	\label{eqn:binom_grassmann_proj}
\end{align}
Note that the generating function $\mu$ is a valid measure only for $\beta_0 > 0$. This translates into the constraints on $\beta$ given in Eq.~\ref{eqn:binom_grassmann_bc} and Eq.~\ref{eqn:binom_grassmann_proj}.

A more general form of binomial kernels can be derived by noting that, if $k(\cdot,\cdot):\mathcal{X} \times \mathcal{X} \to \mathbb{R}_+$ is {\it pd},
then so is $k^\alpha(\cdot,\cdot),~\alpha >0$ (see Proposition 2.7 in Chapter 3 of~\cite{Berg:1984}). This lets us define the Grassmannian kernels
\begin{align}
	k_{bi,bc}(\Mat{X},\Mat{Y}) &= \Big(\beta - \big|\det \big(\Mat{X}^T\Mat{Y} \big)\big|\Big)^{-\alpha}~,~~\beta > 1,~\alpha > 0\;,
	\label{eqn:binom_grassmann_bc2} \\
	k_{bi,p}(\Mat{X},\Mat{Y}) &= \bigg( \beta - \big\|\Mat{X}^T\Mat{Y} \big\|_F^2 \bigg)^{-\alpha}~,~~\beta > p,~\alpha > 0\;. 
	\label{eqn:binom_grassmann_proj2} 
\end{align}

To show that the binomial kernels are universal, we note that 
\begin{equation*}
(1-t)^{-\alpha} = \sum_{j=0}^{\infty} 
\begin{pmatrix}
-\alpha\\
j
\end{pmatrix}
(-1)^jt^j\;,
~~{\rm with} ~~
\begin{pmatrix}
\alpha\\
j
\end{pmatrix}
= \prod_{i=1}^j {(\alpha - i +1)}/{i}\;.
\end{equation*}
It can be seen that $\begin{pmatrix}
-\alpha\\
j
\end{pmatrix}
(-1)^j > 0$, which implies that both $k_{bi,bc}$  and $k_{bi,p}$ have non-negative and full Taylor series. This, as was shown in Corollary 4.57 of~\cite{Steinwart:Book}, is a necessary and sufficient condition for a kernel to be universal.

\subsection{Conditionally Positive Kernels}
\label{sec:cpd_kernel}

Another important class of kernels is the so-called conditionally positive definite kernels~\cite{Berg:1984}. Formally:
\begin{definition}[Conditionally Positive Definite Kernels]
Let $\mathcal{X}$ be a nonempty set. A symmetric function $\psi: \mathcal{X} \times \mathcal{X} \to \mathbb{R}$ is a conditionally
positive definite (\emph{\textbf{cpd}}) kernel on $\mathcal{X}$ if and only if
$\sum_{i,j=1}^nc_ic_jk(x_i,x_j) \geq 0$ for any $n \in \mathbb{N}$, $x_i \in \mathcal{X}$ and $c_i \in \mathbb{R}$ with $\sum_{i=1}^n c_i = 0$.
\end{definition}

The relations between \emph{cpd} kernels and \emph{pd} ones were studied by Berg \etal~\cite{Berg:1984} and Sch\"{o}lkopf~\cite{Scholkopf:NIPS:2001} among others. Before introducing \emph{cpd} kernels on the Grassmannian, we state an important property of \emph{cpd} kernels. 
\begin{proposition}[\cite{Scholkopf:NIPS:2001}]
For a kernel algorithm that is translation invariant, one can equally use \emph{cpd} kernels instead of \emph{pd} ones. 
\end{proposition}

This property relaxes the requirement of having \emph{pd} kernels for certain types of kernel algorithms. A kernel algorithm is translation invariant if it is independent of the position of the origin. 
For example, in SVMs, maximizing the margin of the separating hyperplane between two classes is independent of
the position of the origin. As a result, one can seamlessly use a \emph{cpd} kernel instead of a \emph{pd} kernel in SVMs.
To introduce \emph{cpd} kernels on Grassmannians, we rely on the following proposition:
\begin{proposition}[\cite{Berg:1984}]
If $f: \mathcal{X} \times \mathcal{X} \to \mathbb{R}_+$ is \emph{nd} then $-\log(1 + f)$ is \emph{cpd}.
\end{proposition}

This lets us derive the Grassmannian \emph{cpd} kernels
\begin{align}
	k_{log,bc}(\Mat{X},\Mat{Y}) &= -\log\bigg( 2 - \det\big(\Mat{X}^T\Mat{Y} \big) \bigg)\;,
	\label{eqn:log_grassmann_bc2}\\
	k_{log,p}(\Mat{X},\Mat{Y}) &= -\log\bigg( p + 1 - \big\|\Mat{X}^T\Mat{Y} \big\|_F^2 \bigg)\;. 
	\label{eqn:log_grassmann_proj2}
\end{align}

The ten new kernels derived in this section are summarized in Table~\ref{tab:grassmannian_kernel}.
Note that given the linear  \emph{Pl\"{u}cker} and \emph{projection} kernels, \ie, $k_{lin,bc}(\Mat{X},\Mat{Y}) = |\det(\Mat{X}^T\Mat{Y})|$
and $k_{lin,p}(\Mat{X},\Mat{Y}) = \|\Mat{X}^T\Mat{Y}\|_F^2$, 
it is possible to obtain the polynomial and Gaussian extensions via standard kernel construction rules~\cite{Shawe-Taylor:2004}.
However, our approach lets us derive many other kernels in a principled manner by, \eg, exploiting different measures in Theorem~\ref{thm:laplace_kernel}. Nonetheless, here, we confined ourselves to deriving kernels corresponding to the most popular ones in Euclidean space, and leave the study of additional kernels as future work. 

\section{Experimental Evaluation}%
\label{sec:experiments}%
In this section, we  compare our new kernels with the baseline kernels $k_{bc}^2$ and $k_p$ using three different kernel-based algorithms on Grassmannians: kernel SVM, kernel k-means and kernelized Locality Sensitive Hashing (kLSH). In our experiments, unless stated otherwise, we obtained the kernel parameters (\ie, $\beta$ for all kernels except the logarithm ones 
and $\alpha$ for the polynomial and binomial cases) by cross-validation. 
\subsection{Gender Recognition from Gait}%
We first demonstrate the benefits of our kernels on a binary classification problem on the Grassmannian using SVM and the Grassmannian 
Graph-embedding Discriminant Analysis (GGDA) proposed in~\cite{Harandi:CVPR:2011}.
To this end, we consider the task of gender recognition from gait (\ie, videos of people walking).
We used Dataset-B of the CASIA gait database~\cite{CASIA_DATASET}, which comprises 124 individuals (93 males and 31 females).
The gait of each subject was captured from 11 viewpoints. Every video is represented by a gait energy image (GEI) of size $32 \times 32$ (see Fig.~\ref{fig:CASIA_example}), which has proven effective for gender recognition~\cite{GEI:TIP:2009}.

In our experiment, we used the videos captured with normal clothes and created a subspace of order 3 from the 11 GEIs corresponding to the different viewpoints.
This resulted in 731 points on  $\GRASS{3}{1024}$. We then randomly selected 20 individuals (10 male, 10 female) as training set and used the remaining individuals for testing.
In Table~\ref{tab:table_CASIA_performance}, we report the average accuracies over 10 random partitions.
Note that for the SVM classifier, all new kernels derived from the Pl\"{u}cker embedding outperform $k_{bc}^2$, with highest accuracy obtained with the binomial kernel. Similarly, all new projection kernels outperform $k_p$, and the polynomial kernel achieves the overall highest accuracy of $89.3\%$. For GGDA, bar the case of $k_{log,p}$, all new kernels also outperform previously-known ones.

\def\CASIASIZE {0.08}
\begin{figure}[!tb]
      \centering
        \includegraphics[width=\CASIASIZE\columnwidth,keepaspectratio]{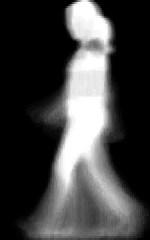}
        \includegraphics[width=\CASIASIZE\columnwidth,keepaspectratio]{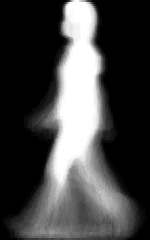}
        \includegraphics[width=\CASIASIZE\columnwidth,keepaspectratio]{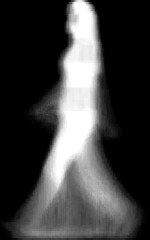}
        \includegraphics[width=\CASIASIZE\columnwidth,keepaspectratio]{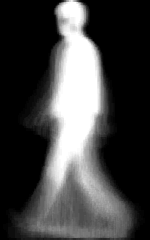}
        \caption{\footnotesize   GEI samples from CASIA~\cite{CASIA_DATASET}.}
        \label{fig:CASIA_example}
\end{figure}

\begin{table}[!tb]
 	\caption    {\footnotesize    
	    \textbf{Gender recognition.} Accuracies on the CASIA gait dataset~\cite{CASIA_DATASET}.
    }
  	\centering
	 \begin{footnotesize}
    \begin{tabular*}{1.0\textwidth}{@{\extracolsep{\fill} } lcccccc}
    	\toprule {\textbf{kernel}}
    	&{$\bf k_{bc}^2$} &{$\bf k_{p,bc}$ } &{$\bf k_{r,bc}$} &{$\bf k_{l,bc}$} &{$\bf k_{bi,bc}$} &{$\bf k_{log,bc}$}\\{\textbf{SVM}}
    	&$76.8\% \pm 9.1$ &$84.1\% \pm 7.2$ &$85.8\% \pm 4.6$ &$84.5\% \pm 4.5$ &$\bf 86.4\% \pm 4.4$ &$82.7\% \pm 7.4$ \\
    	{\textbf{GGDA\cite{Harandi:CVPR:2011}}}
    	&$83.7\% \pm 3.7.$ &$89.0\% \pm 3.7$ &$88.3\% \pm 3.6$ &$88.0\% \pm 3.6$ &$\bf 89.4\% \pm 3.1$ &$84.9\% \pm 3.5$ \\
    	\midrule {\textbf{kernel}}
    	&{$\bf k_{p}$} &{$\bf k_{p,p}$} &{$\bf k_{r,p}$} &{$\bf k_{l,p}$} &{$\bf k_{bi,p}$} &{$\bf k_{log,p}$}\\{\textbf{SVM}}
    	&$83.7\% \pm 4.3$ &$\bf 89.3\% \pm 5.8$ &$88.2\% \pm 5.8$ &$ 87.6\% \pm 5.5$ &$88.7\% \pm 5.1$ &$85.8\% \pm 8.3$ \\
    	{\textbf{GGDA\cite{Harandi:CVPR:2011}}}
    	&$90.3\% \pm 4.7$ &$\bf 93.5\% \pm 2.7$ &$91.3\% \pm 3.8$ &$91.0\% \pm 3.8$ &$ 91.1\% \pm 3.1$ &$89.7\% \pm 3.6$ \\
    	\bottomrule	
    \end{tabular*}
    \label{tab:table_CASIA_performance}
    \end{footnotesize}
\end{table}

\subsection{Pose Categorization}
As a second experiment, we evaluate the performance of our kernels on the task of clustering on the Grassmannian using kernel k-means. 
To this end, we used the CMU-PIE face dataset~\cite{Sim:PAMI:2003}, which contains images of 67 subjects with 13 different poses and 21 different illuminations (see Fig.~\ref{fig:PIE_samples} for examples). 
From each image, we computed a $2\times 2$ spatial pyramid of LBP~\cite{Ojala:PAMI:2002} histograms and concatenated them to form a $232$ dimensional descriptor. For each subject, we collected the images acquired with the same pose, but different illuminations, in an image set, which we then represented as a linear subspace of order 3. This resulted in a total of $67 \times 13 = 871$ Grassmannian points on $\GRASS{3}{232}$. We used 10 samples from each pose to compute the kernel parameters. 

\def \LEFT_MICE_PER {0.30}
\def \RIGHT_POSE_PER {0.8}
\begin{figure}[!tb]
		\begin{minipage}{0.50 \textwidth}
    \centering
 	  \includegraphics[width=\RIGHT_POSE_PER \textwidth]{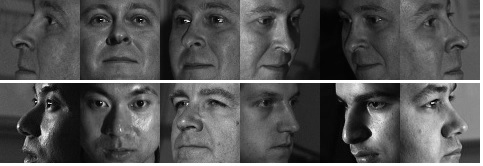} 
      \caption{\footnotesize Sample images from CMU-PIE.}
      \label{fig:PIE_samples}
	\end{minipage}
	\hspace{2ex}	
	\begin{minipage}{0.45 \textwidth}
	\centering
	\includegraphics[width= \LEFT_MICE_PER \textwidth]{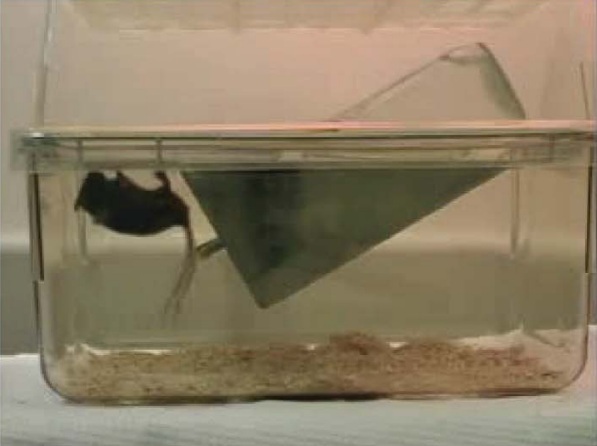}%
	\includegraphics[width= \LEFT_MICE_PER \textwidth]{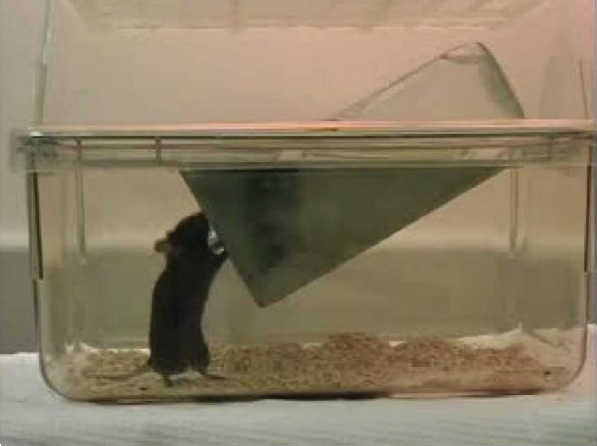}%
	\includegraphics[width= \LEFT_MICE_PER \textwidth]{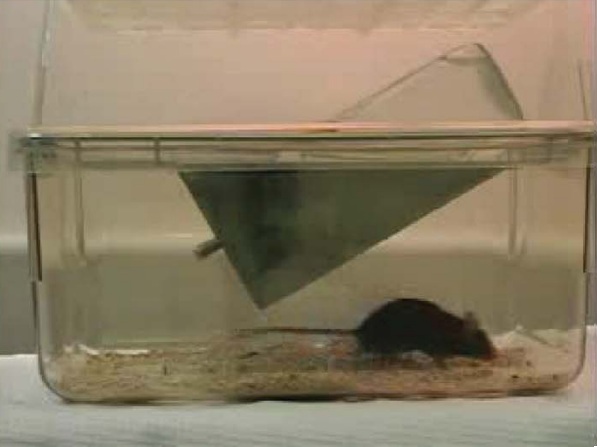}
	\caption{Sample images from the mouse behavior dataset~\cite{Mouse:Dataset}.}
	\label{fig:Mouse_samples}
	\end{minipage}
\end{figure}

The goal here is to cluster together image sets representing the same pose. To evaluate the quality of the clusters, we report both the clustering accuracy and the Normalized Mutual Information (NMI)~\cite{Strehl:AAAI:2000}, which measures the amount of statistical information shared by random variables representing the cluster distribution and the underlying class distribution of the data points.
From the results given in Table~\ref{tab:table_PIE_performance},  we can see that, with the exception of $k_{log,p}$, the new kernels in each embedding outperform their respective baseline, $k_p$ or $k^2_{bc}$. For the Binet-Cauchy kernels, the maximum accuracy (and NMI score) is reached by the RBF kernel. The overall maximum accuracy of $82.9\%$ is achieved by the projection-based binomial kernel.

We also evaluated the intrinsic k-means algorithm of~\cite{Turaga:PAMI:2011}. This algorithm achieved $67.7\%$ accuracy and an NMI score of $0.75$. Furthermore, intrinsic k-means required $9766s$ to perform clustering on an i7 machine using Matlab. On the same machine, the runtimes for kernel k-means using $k_{r,bc}$ and $k_{bi,p}$ (which achieve the highest accuracies in Table~\ref{tab:table_PIE_performance}) 
were $3.1s$ and $2.5s$, respectively. This clearly demonstrates the benefits of RKHS embedding to tackle clustering problems on the Grassmannian.

\begin{table}[!tb]
 	\caption    { \small   
	    \textbf{Pose categorization.} Clustering accuracies on the CMU-PIE dataset.
    }
  	\centering
    \begin{tabular*}{1.0\textwidth}{@{\extracolsep{\fill} } lcccccc}
    	\toprule {\textbf{kernel}}
    							&{$\bf k^2_{bc}$} &{$\bf k_{p,bc}$} &{$\bf k_{r,bc}$} &{$\bf k_{l,bc}$} &{$\bf k_{bi,bc}$} &{$\bf k_{log,bc}$}\\
    	{\textbf{accuracy}}  	&$70.3\%$ 		 &$72.2\% $ 		   &{$\bf 78.9\%$}	 &$74.8\%$ 		   &$78.5\%$ 		  &$72.2\% $ \\
	   	{\textbf{NMI}}  		&$0.763$ 		 &$0.779$ 			   &{$\bf 0.803$}	 &$0.786$ 		   &$0.798$ 		  &$0.772$ \\
    	\midrule {\textbf{kernel}}
    							&{$\bf k_{p}$} &{$\bf k_{p,p}$} &{$\bf k_{r,p}$} &{$\bf k_{l,p}$} &{$\bf k_{bi,p}$} &{$\bf k_{log,p}$}\\
    	{\textbf{accuracy}}    	&$77.1\%$ 		&$79.9\%$ 		  &$80.9\%$ 		&$79.8\%$ 		 &$\bf 82.9\%$ 		   &$74.4\%$ \\
	   	{\textbf{NMI}}  		&$0.810\%$ 		&$0.817$ 		  &$0.847$ 			&$0.843$ 		 &$\bf 0.853$ 		   &$0.812$ \\
	   	\bottomrule	
    \end{tabular*}
    \label{tab:table_PIE_performance}
\end{table}

\subsection{Mouse Behavior Analysis}

Finally, we utilized kernelized Locality-Sensitive Hashing (kLHS)~\cite{Kulis:PAMI:2012} to perform recognition on the 2000 videos of the mice behavior dataset~\cite{Mouse:Dataset}. 
The basic idea of kLSH is to search for a projection from an RKHS to a low-dimensional Hamming space, where each sample is encoded with a $b$-bit vector called the hash key. The approximate nearest-neighbor to a query can then be found efficiently in time sublinear in the number of training samples.

The mice dataset~\cite{Mouse:Dataset} contains 8 behaviors (\ie, \emph{drinking}, \emph{eating}, \emph{grooming}, \emph{hanging}, \emph{rearing},
\emph{walking}, \emph{resting} and \emph{micro-movement of head}) of several mice with different coating colors, sizes and genders (see Fig.~\ref{fig:Mouse_samples} for examples). In each video, we estimated the background to extract the region containing the mouse in each frame. These regions were then resized to $48\times 48$, and the video represented with an order 6 subspace, thus yielding points on $\GRASS{6}{2304}$.
We randomly chose 1000 videos for training and used the remaining 1000 videos for testing. We report the average recognition accuracy over 10 random partitions.

Fig.~\ref{fig:Mouse_results} depicts the recognition accuracies of the new and baseline kernels as a function of the number of bits $b$.
For the Pl\"{u}cker embedding kernels, the gap between our RBF kernel and $k_{bc}^2$ reaches $23\%$ for a hash key of size 30. For the same hash key size, the projection-based heat kernel outperforms $k_p$ by more than $14\%$, and thus reaches the overall highest accuracy of $67.2\%$.

\def \MICE_PER {0.450}
\begin{figure}[!tb]
    \centering
    \begin{subfigure}[b]{\MICE_PER \textwidth}
    	\includegraphics[width=\textwidth]{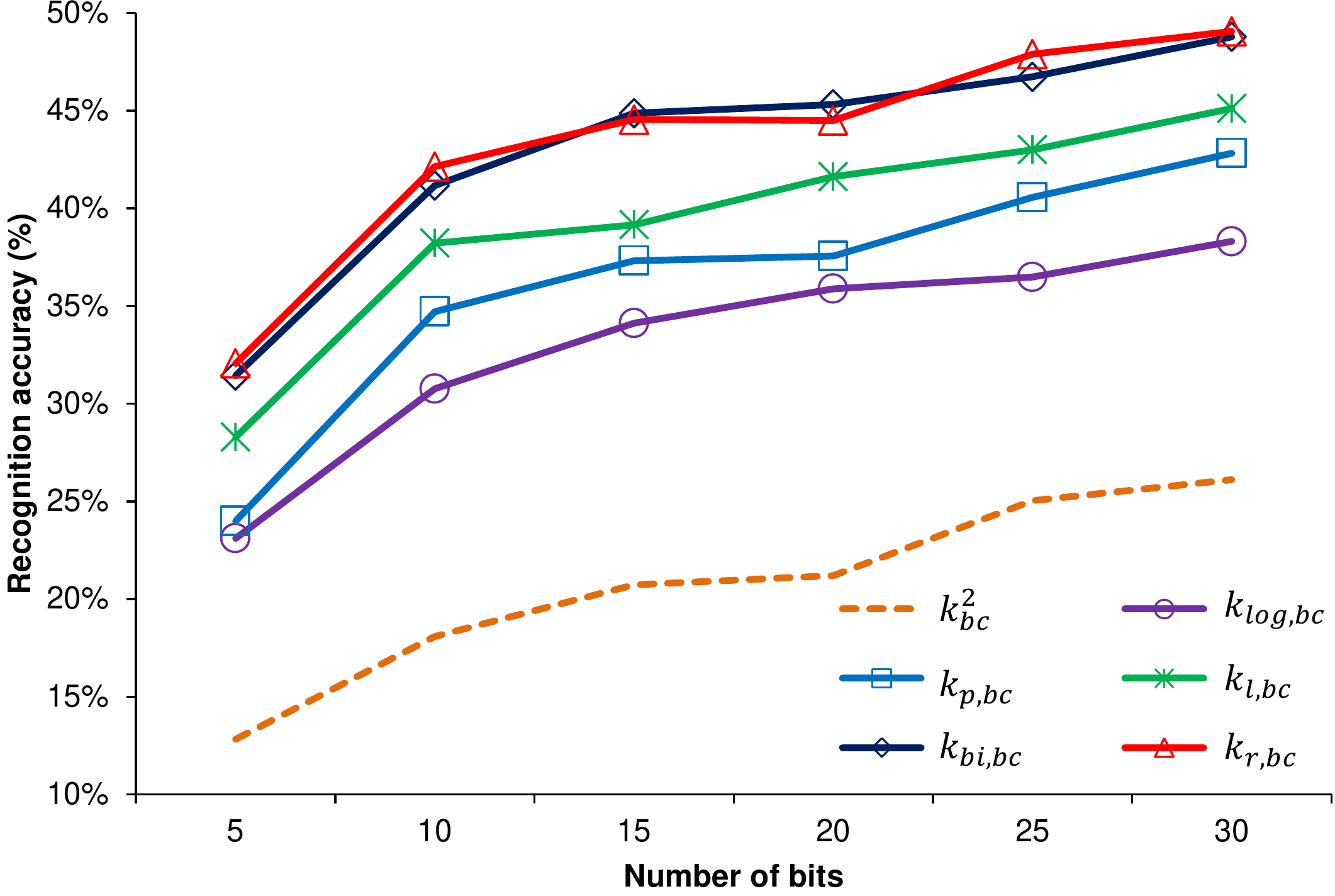}
    \end{subfigure} %
    \begin{subfigure}[b]{\MICE_PER \textwidth}
    	\includegraphics[width=\textwidth]{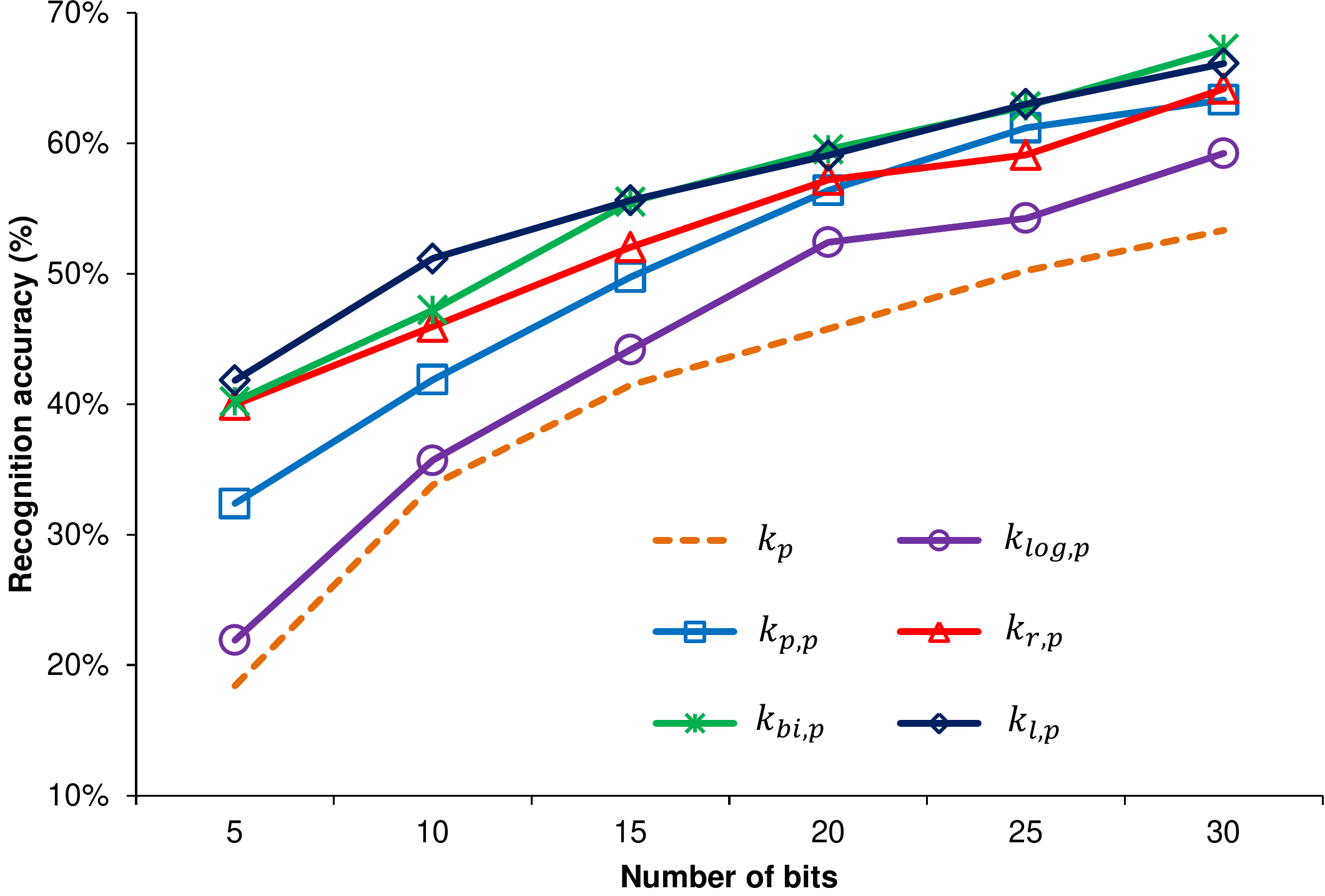}
    \end{subfigure}       
    \caption{{\small \textbf{Video hashing.}} {\small Approximate nearest-neighbor accuracies for kernels derived from \textbf{(left)}  
    the Pl\"{u}cker and \textbf{(right)} 
     the projection embeddings.}}
	\label{fig:Mouse_results}
\end{figure}

\subsection{Kernel Sparse Coding}

We performed an experiment on body-gesture recognition using the UMD Keck dataset~\cite{Lin:ICCV:2009}. 
To this end, we consider the problem of kernel sparse coding on the Grassmannian which 
can be formulated as
{\small
\begin{align}
	\underset{\Vec{y}}{\min} \:  \left\| \phi(\Mat{X}) - \sum_{j=1}^{N}\nolimits y_{j} \phi(\Mat{D}_j) \right\|^2 +\lambda \|\Vec{y}\|_1\;,
    \label{eqn:SC_Grass}
\end{align}
}%
where $\Mat{D}_j \in \GRASS{p}{d}$ is a dictionary atom, $\Mat{X} \in \GRASS{p}{d}$ is the query and $\vec{y}$ is the vector of sparse codes. 
In practice, we used each training sample as an atom in the dictionary. Note that, as shown in~\cite{Gao:ECCV:2010},~\eqref{eqn:SC_Grass} only depends on the kernel values computed between the dictionary atoms, as well as between the query point and the dictionary. Classification is then performed by assigning the label of the dictionary element $\Mat{D}_i$ with strongest response $\Vec{y}_i$ to the query.

The UMD Keck dataset~\cite{Lin:ICCV:2009} comprises 14 body gestures with static and dynamic backgrounds 
(see examples in Figure~\ref{fig:Keck_samples}). The dataset contains 126 videos from static scenes and 168 ones from
dynamic environments.  Following the experimental protocol used in~\cite{Lui:JMLR:2012}, we first extracted the region of interest around each gesture and resized it to $32 \times 32$ pixels. We then represented each video by
a subspace of order 6, thus yielding points on $\GRASS{6}{1024}$. 

\def \KECK_SIZE {0.30}
\begin{figure}[!b]
	\centering
    \begin{subfigure}[b]{\KECK_SIZE \textwidth}
    	\includegraphics[width=\textwidth]{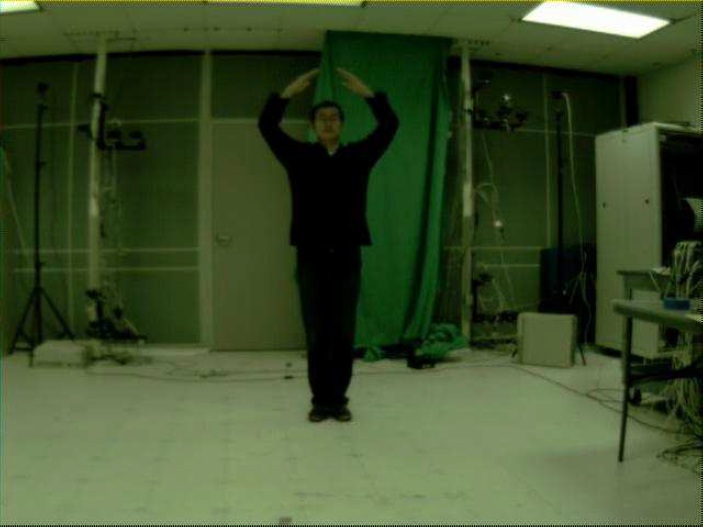}
    \end{subfigure} %
    \begin{subfigure}[b]{\KECK_SIZE \textwidth}
    	\includegraphics[width=\textwidth]{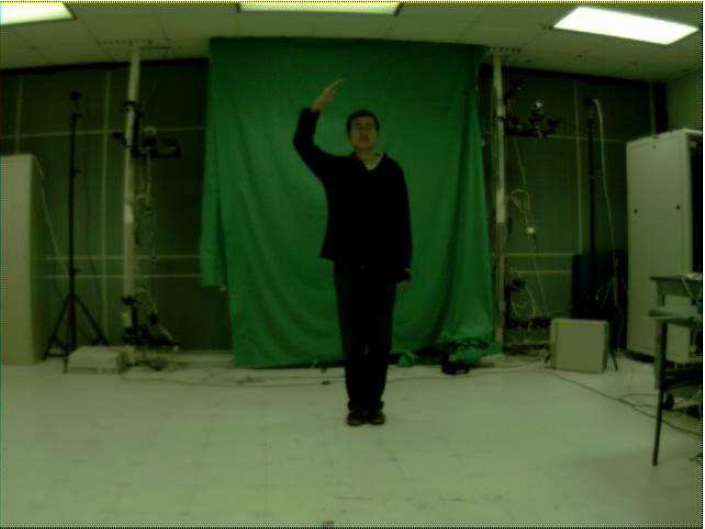}
    \end{subfigure} %
        \begin{subfigure}[b]{\KECK_SIZE \textwidth}
    	\includegraphics[width=\textwidth]{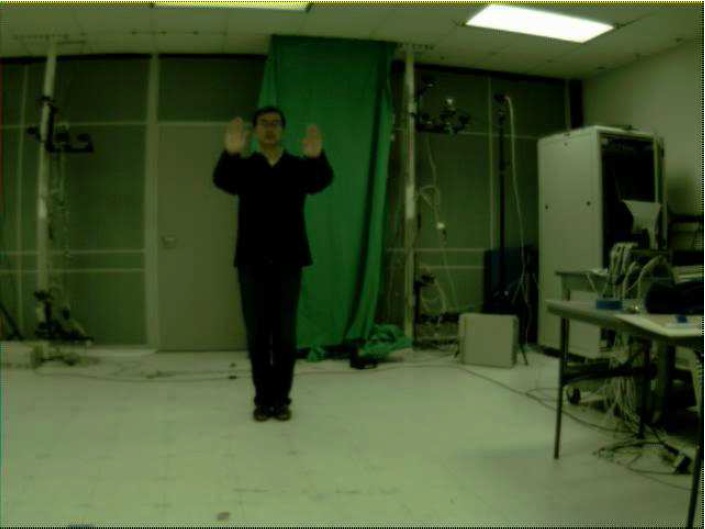}
    \end{subfigure}       
    \caption{{\small Sample images from the UMD Keck body-gesture dataset~\cite{Lin:ICCV:2009}.}}
	\label{fig:Keck_samples}	
\end{figure}

Table~\ref{tab:table_KECK_performance_SC} compares the performance of our kernels
with that of $k_{bc}^2$ and $k_{p}$. Note that our kernels outperform the baselines in both the static and dynamic settings.
The maximum accuracy is obtained by $k_{l,p}$ for the static scenario ($99.2\%$), and by $k_{bi,p}$ for the dynamic one ($99.1\%$).
For the same experiments, the state-of-the-art solution using product manifolds~\cite{Lui:JMLR:2012} achieves $94.4\%$ and $92.3\%$, 
respectively. 
\begin{table}[!tb]
 	\caption    {\footnotesize    
	    \textbf{Body-gesture recognition.} Accuracies on UMD Keck~\cite{Lin:ICCV:2009} using kernel sparse coding.
    }
  	\centering
    \begin{tabular*}{1.0\textwidth}{@{\extracolsep{\fill} } lcccccc}
    	\toprule {\textbf{kernel}}
    	&{$\bf k_{bc}^2$} &{$\bf k_{p,bc}$ } &{$\bf k_{r,bc}$} &{$\bf k_{l,bc}$} &{$\bf k_{bi,bc}$} &{$\bf k_{log,bc}$}\\{\textbf{Static}}
    	&$84.1\% \pm 3.6$ &$88.1\% \pm 6.0$ &$90.5\% \pm 4.7$ &$91.3\% \pm 7.7$ &$\bf 93.7\% \pm 3.6$ &$87.3\% \pm 3.6$ \\{\textbf{Dynamic}}
    	&$90.2\%$ &$92.0\%$ &$93.8\%$ &$\bf 94.6\%$ &$\bf 94.6\%$ &$92.9\%$\\
    	\midrule {\textbf{kernel}}
    	&{$\bf k_{p}$} &{$\bf k_{p,p}$} &{$\bf k_{r,p}$} &{$\bf k_{l,p}$} &{$\bf k_{bi,p}$} &{$\bf k_{log,p}$}\\{\textbf{Static}}
    	&$88.9\% \pm 8.4$ &$94.4\% \pm 6.0$ &$97.6\% \pm 4.1$ &$\bf 99.2\% \pm 1.3$ &$96.0\% \pm 3.6$ &$92.9\% \pm 2.4$ \\{\textbf{Dynamic}}
    	&$91.1\%$ &$92.0\%$ &$98.2\%$ &$98.2\%$ &$\bf 99.1\%$ &$97.3\%$\\
    	\bottomrule	
    \end{tabular*}
    \label{tab:table_KECK_performance_SC}
\end{table}

\section{Conclusions and Future Work}
We have introduced a set of new positive definite kernels to embed Grassmannian into Hilbert spaces, which have a more familiar Euclidean structure. This set includes, among others, universal Grassmannian kernels, which have the ability to approximate general functions. Our experiments have demonstrated the superiority of such kernels over previously-known Grassmannian kernels, \ie, the Binet-Cauchy kernel~\cite{Wolf:JMLR:2003} and the projection kernel~\cite{Hamm:ICML:2008}. It is important to keep in mind, however, that choosing the right kernel for the data at hand remains an open problem. In the future, we intend to study if searching for the best probability measure in Theroem~4 could give a partial answer to this question.

\appendix
\section{Proof of Length Equivalence}
\label{sec:proof_app}
Here, we prove Theorem~2 from Section~3,
\ie, the equivalence up to a scale of $\sqrt{2}$ of the length of any given curve under the Binet-Cauchy distance $\delta_{bc}$ derived from the Pl\"ucker embedding and the geodesic distance $\delta_g$.
The proof of this theorem follows several steps.
We start with the definition of curve length and intrinsic metric.
Without any assumption on differentiability, let $(\mathcal{M},d)$ be a metric space.  
A curve in $\mathcal{M}$ is a continuous function $\gamma : [0, 1] \rightarrow \mathcal{M}$ and joins the starting point $\gamma(0) = x$
to the end point $\gamma(1) = y$. 
\begin{definition} \label{def:curve_length}
	The length of a curve $\gamma$ is the supremum of $L(\gamma ; \{t_i \})$ over all possible partitions \mbox{$\{t_i \}$},
	where \mbox{$0 = t_0 < t_1 < \cdots < t_{n-1} < t_n = 1$} and
	\mbox{{$L(\gamma ; \{t_i \}) = \sum_{i}d\left(\gamma(t_i),\gamma(t_{i-1})\right)$}}.
\end{definition}

\begin{definition} \label{def:intrinsic_metric_thm}
	The intrinsic metric $\widehat{\delta}(x,y)$ on $\mathcal{M}$
	is defined as the infimum of the lengths of all paths from $x$ to $y$.
\end{definition}
	
\begin{theorem}[~\cite{Hartley:IJCV:2013}] \label{thm:intrinsic_metric_thm}
	If the intrinsic metrics induced by two metrics $d_1$ and $d_2$ are identical up to a scale $\xi$,
	then the length of any given curve is the same under both metrics up to $\xi$.
\end{theorem}
\begin{theorem}[~\cite{Hartley:IJCV:2013}]
	If $d_1(x,y)$ and $d_2(x,y)$ are two metrics defined on a space $\mathcal{M}$ such that
	\begin{equation}
		\lim_{d_1(x,y) \rightarrow 0} \:\frac{d_2(x,y)}{d_1(x,y)} = 1.
		\label{eqn:intrinsic_metric0}
	\end{equation}
	uniformly (with respect to $x$ and $y$), then their intrinsic metrics are identical.
\end{theorem}
Therefore, here, we need to study the behavior of 
\begin{equation*}
	\lim_{\delta_{g}^2(\Mat{X},\Mat{Y}) \rightarrow 0} \:\frac{\delta_{bc}^2(\Mat{X},\Mat{Y})}{\delta_{g}^2(\Mat{X},\Mat{Y})}\;
\end{equation*}
to prove our theorem on curve length equivalence.
	
\begin{proof}
Since $\sin(\theta_i) \rightarrow \theta_i$ for $\theta_i \rightarrow 0$, we can see that 
\begin{align*}
\lim_{\delta_g(\Mat{X},\Mat{Y}) \rightarrow 0} \:\frac{\delta_{bc}^2(\Mat{X},\Mat{Y})}{\delta_g^2(\Mat{X},\Mat{Y})} 
&= 
\lim_{\theta_i \rightarrow 0} \:\frac{2-2\prod_{i=1}^{p}(1 - \sin^2\theta_i)}{\sum_{i=1}^{p}\theta_i^2}\\ 
&= 
\lim_{\theta_i  \rightarrow 0} \:\frac{2-2\prod_{i=1}^{p}(1 - \theta_i^2)}{\sum_{i=1}^{p}\theta_i^2} 
= 
\lim_{\theta_i  \rightarrow 0} \:\frac{2-2(1 - \sum_{i=1}^{p}\theta_i^2)}{\sum_{i=1}^{p}\theta_i^2} 
=2\;.
\end{align*}
This, in conjunction with Theorem~\ref{thm:intrinsic_metric_thm}, concludes the proof.\qed
\end{proof}

\end{document}